\documentclass[letterpaper, 10 pt, conference]{ieeeconf}  

\IEEEoverridecommandlockouts                              

\usepackage{multirow}
\usepackage{graphicx} 

\usepackage{caption}
\usepackage{subcaption}

\usepackage{xcolor}
\usepackage{algorithm}
\usepackage[noend]{algpseudocode} 
\usepackage{amssymb,amsmath,comment,cite}
\newtheorem{theorem}{Theorem}
\newtheorem{lemma}{Lemma}
\newtheorem{definition}{Definition}

\usepackage{ifthen}
\newboolean{shortver}
\setboolean{shortver}{true}

\title{\LARGE \bf
	Multi-objective Conflict-based Search Using Safe-interval Path Planning
}

\author{Zhongqiang Ren$^{1}$, Sivakumar Rathinam$^{2}$, Maxim Likhachev$^{1}$ and Howie Choset$^{1}$
	\thanks{$^{1}$ Zhongqiang Ren, Maxim Likhachev and Howie Choset are with Carnegie Mellon University, 5000 Forbes Ave., Pittsburgh, PA 15213, USA. 
	}%
	\thanks{$^{2}$Sivakumar Rathinam is with Texas A\&M University,
		College Station, TX 77843-3123.
	}
}

\begin{document}
	
	\maketitle
	
		\thispagestyle{empty}
		\pagestyle{empty}
		\thispagestyle{plain}
		\pagestyle{plain}
		\pagenumbering{arabic}
	
	\begin{abstract}
		This paper addresses a generalization of the well known multi-agent path finding (MAPF) problem that optimizes multiple conflicting objectives simultaneously such as travel time and path risk. This generalization, referred to as multi-objective MAPF (MOMAPF), arises in several applications ranging from hazardous material transportation to construction site planning. In this paper, we present a new multi-objective conflict-based search (MO-CBS) approach that relies on a novel multi-objective safe interval path planning (MO-SIPP) algorithm for its low-level search. We first develop the MO-SIPP algorithm, show its properties and then embed it in MO-CBS. We present extensive numerical results to show that (1) there is an order of magnitude improvement in the average low level search time, and (2) a significant improvement in the success rates of finding the Pareto-optimal front can be obtained using the proposed approach in comparison with the previous MO-CBS.

	\end{abstract}
	
	
	\graphicspath{{./figures/}}
	
	\section{Introduction}\label{sec:intro}
	Multi-agent path finding (MAPF), as its name suggests, computes an ensemble of collision-free paths for multiple agents between their respective start and goal locations.
Conventional MAPF problems~\cite{stern2019multi} typically consider optimizing a single path criterion such as path length or travel time.
\ifthenelse{\boolean{shortver}}{%
	However, in many real-world planning applications~\cite{montoya2013multiobjective,xu2021multi}, multiple conflicting objectives such as path length, travel risk and other domain-specific metrics are simultaneously optimized.
}{
	However, in many real-world planning applications~\cite{montoya2013multiobjective,xu2021multi,hayat2017multi,erkut2007hazardous}, multiple conflicting objectives such as path length, travel risk and other domain-specific metrics are simultaneously optimized.
}
When multiple objectives cannot be readily converted into a single weighted objective, multi-objective planners\cite{ehrgott2005multicriteria} that aim to find a set of Pareto-optimal solutions are required.
A solution is Pareto-optimal if there exists no other solution that will yield an improvement in one objective without causing a deterioration in at least one of the other objectives.
Finding a Pareto-optimal set for multi-objective MAPF (MOMAPF) problems while ensuring conflict-free paths for agents in each solution is quite challenging as the size of the Pareto-optimal set may grow exponentially with respect to the number of agents as well as the dimension of the search space\cite{serafini1987some,yu2013structure_nphard}.
In this article, we propose a new multi-objective conflict-based search (MO-CBS) approach to find the Pareto-optimal set for MOMAPF.

Conflict-based search (CBS)~\cite{sharon2015conflict} poses MAPF as a graph search problem and computes an optimal solution for agents with respect to a single objective.
CBS is a two-level search algorithm where on the high level, collisions between agents are detected and constraints are generated from these collisions and added to the low level search.
On the low level, a single-agent planner, such as A*, is invoked to plan paths in a time-augmented graph to satisfy all the added constraints. 
Leveraging both CBS and multi-objective dominance~\cite{ehrgott2005multicriteria}, multi-objective CBS (MO-CBS) has been proposed in our prior work~\cite{ren2021multi} to compute a Pareto-optimal set of solutions for MOMAPF.
MO-CBS employs a similar two-level search workflow, where on the low level, an A*-like multi-objective path planner, such as NAMOA*~\cite{mandow2008multiobjective}, is used to search over a time-augmented graph to compute Pareto-optimal individual paths for an agent subject to constraints.

We learnt from our prior work~\cite{ren2021multi} that using NAMOA* for the low level search over a time-augmented graph is inefficient: First, optimal search over time augmented graphs are time consuming because of the inclusion of the time dimension~\cite{phillips2011sipp}; Second, the number of Pareto-optimal paths in a time-augmented graph can grow exponentially with respect to the number of nodes to be searched~\cite{hansen1980bicriterion}.
Since the low level search is repeatedly invoked in any CBS-based methods, using an inefficient algorithm at the low level can significantly burden a CBS-based algorithm, which includes MO-CBS.
This work aims to address this issue by developing a new low level search algorithm called multi-objective safe-interval path planning (MO-SIPP).

The proposed MO-SIPP leverages SIPP~\cite{phillips2011sipp} to search the time dimension efficiently while optimizing multiple objectives.
We first show that the MO-SIPP is able to compute all Pareto-optimal solutions for a single agent.
Then, we employ MO-SIPP as the low level planner for MO-CBS and verify our idea using an MAPF benchmark set~\cite{stern2019multi}.
From our numerical results obtained using the proposed approach for instances up to 3 objectives, we observed (1) an order of magnitude improvement in the average low level search time, and (2) around 20\% improvement in the success rates of finding Pareto-optimal solutions within a fixed time limit in some of the maps.

The rest of the article discusses the related work in Sec.~\ref{sec:related}.
A review of SIPP and the proposed MO-SIPP are presented in Sec.~\ref{sec:mosipp}.
MO-CBS and its improved version via MO-SIPP is discussed in Sec.~\ref{sec:mocbssipp} with numerical results in Sec.~\ref{sec:result}.
Finally, we conclude and outline the future work in Sec.~\ref{sec:conclude}.

	
%

	\section{Related Work}\label{sec:related}
	\ifthenelse{\boolean{shortver}}{%
		
\subsection{Multi-objective Path Planning}
Existing approaches for {\it single-agent}, multi-objective path planning (MOPP) problems compute an exact or approximated set of Pareto-optimal paths for the agent between its start and goal locations with respect to \emph{multiple} objectives.
One common approach to solve a MOPP is to weight the multiple objectives and transform it to a single-objective problem~\cite{emmerich2018tutorial,ehrgott2005multicriteria}.
The transformed problem can then be solved using any single-objective algorithm.
This approach, however, requires in-depth domain knowledge to design the weighting procedure; it may also requires one to repeatedly solve the transformed single-objective problem for different sets of weights in order to capture the Pareto-optimal set which is quite challenging \cite{marler2004survey}.

Additionally, MOPP has been solved directly via graph search techniques~\cite{moastar,mandow2008multiobjective,ulloa2020simple} and evolutionary algorithms~\cite{weise2020momapf} where a Pareto-optimal set of solutions is computed exactly or approximated.
These graph-based approaches provide guarantees about finding all Pareto-optimal solutions but can run slow for hard cases, where the number of Pareto-optimal solutions is large. 
 MO-CBS belongs to this category of search techniques that directly computes a Pareto-optimal set with quality guarantees.

\subsection{Multi-agent Path Finding}
Various methods have been developed to compute an optimal solution for multi-agent path finding (MAPF) problems~\cite{goldenberg2014enhanced,wagner2015subdimensional,surynek2016efficient,lam2019bcp,sharon2015conflict}.
In addition, different variants of MAPF have also been considered~\cite{andreychuk2019multi,ren2021loosely,ren2021ms,honig2018conflict,ma2019lifelong,ma2016optimal,cohen2019optimal}, to name a few.
However, all these methods optimizes a single objective defined over paths.

For multi-objective MAPF (MOMAPF), evolutionary algorithms~\cite{weise2020momapf} have been leveraged to solve a variant of MOMAPF where agents are not allowed to wait in place and collisions between the agent's paths are modeled in one of the objectives and not as a constraint.
Recently, by leveraging M*~\cite{wagner2015subdimensional} and CBS~\cite{sharon2015conflict} respectively, MOM* \cite{ren2021subdimensional} and MO-CBS \cite{ren2021multi} have been proposed to solve the MOMAPF.
As aforementioned, in this paper, we propose a new low level search algorithm called MO-SIPP and embed it in MO-CBS. This new approach outperforms the standard MO-CBS in all cases in terms of success rates in finding the Pareto-optimal set within a fixed time limit.

\subsection{Safe Interval Path Planning}
Safe interval path planning (SIPP)~\cite{phillips2011sipp} was originally developed to compute a \emph{single-agent} collision-free trajectory from a start to a goal location while minimizing the arrival time, in an environment with dynamic obstacles moving along known trajectories.
SIPP has been extended in several directions in the literature, such as sub-optimal SIPP~\cite{phillips2011planning,yakovlev2020revisiting}, anytime SIPP~\cite{narayanan2012anytime}, generalized SIPP~\cite{gonzalez2012using}, etc.

On the other hand, SIPP has also been used for several related problems such as pickup and delivery problems~\cite{ma2019lifelong} and any-angle path finding~\cite{yakovlev2017any}. With respect to incorporating SIPP into CBS, recent work in \cite{andreychuk2019multi} proposes a new method called Continuous-time CBS which aims to handle agents moving at different speeds. ECBS-CT~\cite{cohen2019optimal} extends SIPP and CBS to solve a multi-agent motion planning problem which computes kinodynamically feasible paths for agents.
However, to our limited knowledge, no existing work has leveraged SIPP to minimize multiple objectives.
The proposed multi-objective SIPP (MO-SIPP) in this work takes a first step to fill this gap.
Leveraging the proposed MO-SIPP, we then achieve an order of magnitude speed up in the low level search of MO-CBS.

	}{
		
\subsection{Multi-objective Path Planning}
Existing approaches for {\it single-agent}, multi-objective path planning (MOPP) problems compute an exact or approximated set of Pareto-optimal paths for the agent between its start and goal locations with respect to \emph{multiple} objectives.
The applications of MOPP can be found in construction site routing~\cite{soltani2004fuzzy}, hazardous material transportation~\cite{erkut2007hazardous}, and others~\cite{montoya2013multiobjective,xu2021multi}.
One common approach to solve a MOPP is to weight the multiple objectives and transform it to a single-objective problem~\cite{emmerich2018tutorial,ehrgott2005multicriteria}. The transformed problem can then be solved using any single-objective algorithm. This approach, however, requires in-depth domain knowledge to design the weighting procedure; it may also requires one to repeatedly solve the transformed single-objective problem for different sets of weights in order to capture the Pareto-optimal set which is quite challenging \cite{marler2004survey}.



Additionally, MOPP has been solved directly via graph search techniques~\cite{moastar,mandow2008multiobjective,ulloa2020simple} and evolutionary algorithms~\cite{weise2020momapf} where a Pareto-optimal set of solutions is computed exactly or approximated.
These graph-based approaches provide guarantees about finding all Pareto-optimal solutions but can run slow for hard cases, where the number of Pareto-optimal solutions is large. 
MO-CBS belongs to this category of search techniques that directly computes a Pareto-optimal set with quality guarantees.

\subsection{Multi-agent Path Finding}
Various methods have been developed to compute an optimal solution for multi-agent path finding (MAPF) problems including A*-based approaches~\cite{standley2010finding,goldenberg2014enhanced}, subdimensional expansion~\cite{wagner2015subdimensional}, compilation-based solver~\cite{surynek2016efficient}, integer programming-based methods~\cite{lam2019bcp} and conflict-based search (CBS)~\cite{sharon2015conflict}.
In addition, different variants of MAPF have also been considered, such as agents moving with different speeds~\cite{andreychuk2019multi,ren2021loosely}, agents moving with stochastic travel times~\cite{peltzer2020stt}, visiting multiple goals along the path~\cite{ren2021ms,honig2018conflict}, pickup-and-delivery tasks~\cite{ma2019lifelong,ma2016optimal}, satisfying kinodynamic constraints~\cite{cohen2019optimal} to name a few.
However, all these methods optimize a single objective defined over paths.

For multi-objective MAPF (MOMAPF), evolutionary algorithms~\cite{weise2020momapf} have been leveraged to solve a variant of MOMAPF where agents are not allowed to wait in place and collisions between the agent's paths are modeled in one of the objectives and not as a constraint.
Recently, by leveraging M*~\cite{wagner2015subdimensional} and CBS~\cite{sharon2015conflict} respectively, MOM* \cite{ren2021subdimensional} and MO-CBS \cite{ren2021multi} have been proposed to solve the MOMAPF. 
As we mentioned earlier, in this paper, we propose a new low level search algorithm called MO-SIPP and embed it in MO-CBS. This new approach outperforms the standard MO-CBS in all cases in terms of success rates in finding the Pareto-optimal set within a fixed time limit.

\subsection{Safe Interval Path Planning}
Safe interval path planning (SIPP)~\cite{phillips2011sipp} was originally developed to compute a \emph{single-agent} collision-free trajectory from a start to a goal location while minimizing the arrival time, in an environment with dynamic obstacles moving along known trajectories.
SIPP, as a fast variant of A*, uses safe-intervals rather than time steps to represent the time dimension. This approach significantly reduces the size of the search space, and thus improves search efficiency (in comparison to applying A* over the entire time-augmented graph).

SIPP has been extended in several directions in the literature. To name a few, sub-optimal SIPP~\cite{phillips2011planning,yakovlev2020revisiting} trades off between search efficiency and solution quality.
Anytime SIPP~\cite{narayanan2012anytime} begins by computing a sub-optimal feasible solution quickly at first and then improves the solution quality until the allocated planning time runs out.
GSIPP~\cite{gonzalez2012using} generalizes SIPP to minimize an objective other than arrival time. SIPP has also been used for pickup and delivery problems~\cite{ma2019lifelong} and any-angle path finding~\cite{yakovlev2017any}, etc. With respect to incorporating SIPP into CBS, recent work in \cite{andreychuk2019multi} proposes a new method called Continuous-time CBS which aims to handle agents moving at different speeds. ECBS-CT~\cite{cohen2019optimal} extends SIPP and CBS to solve a multi-agent motion planning problem which computes kinodynamically feasible paths for agents. No existing work we are aware of has leveraged SIPP to minimize multiple objectives. The proposed multi-objective SIPP (MO-SIPP) in this work takes a first step to fill this gap. Leveraging the proposed MO-SIPP, we then achieve an order of magnitude speed up in the low level search of MO-CBS.

	}

	\section{Multi-objective Safe-Interval Path Planning}\label{sec:mosipp}
	
\subsection{Preliminaries}
We begin with a description about the problem solved by safe-interval path planning (SIPP) and then provide a summary of SIPP.
Given a graph $G=(V,E)$, let $G^t = (V^t, E^t) = G \times \{0,1,\dots,T\}$ denote a time-augmented graph of $G$, where each vertex $v \in V^{t}$ is defined as $v=(u,t), u\in V, t \in \{0,1,\dots,T\}$ and $T$ is a pre-defined time horizon, which is typically a large positive integer. Edges within $G^{t}$ is represented as $E^{t}= V^{t} \times V^{t}$ where $(u_1,t_1),(u_2,t_2)$ is connected in $G^t$ if $(u_1,u_2) \in E$ and $t_2=t_1 + 1$.
Wait in place is also allowed in $G^{t}$ which means $(u,t),(u,t+1), u \in V$ is connected in $G^t$.
A dynamic obstacle with a known trajectory is represented as a set of subsequently occupied nodes in $G^t$.
An illustrative example of $G$, $G^t$ and dynamic obstacles can be found in Fig.~\ref{fig:sipp_eg}.
For the rest of the article, when a node or edge is mentioned, we point out to which graph ($G$ or $G^t$) it belongs if needed.
Given an initial node $v_{init}$ and goal node $v_{goal}$ in $G$, SIPP aims to plan a collision-free trajectory $\tau$ from $v_{init}$ at time zero to $v_{goal}$ with the minimum arrival time.

To solve the problem, A* can be applied to search over $G^t$ to find a collision-free trajectory with the minimum arrival time.
However, this approach is very inefficient as the time dimension is searched in a step-by-step manner.
To overcome this challenge, SIPP~\cite{phillips2011sipp} compresses time steps into intervals.
Let tuple $s=(v,[t_a,t_b])$ denote a search state at node $v\in G$ where $t_a,t_b$ are the beginning and ending time steps of a safe (time) interval respectively.
A safe interval is a maximal contiguous time range in which a node is not occupied by any dynamic obstacles. 
State $s=(v,[t_a,t_b])$ indicates that node $v$ is not occupied by any dynamic obstacle and hence the name safe interval for $[t_a,t_b]$.
Note that, any two safe intervals at the same node never intersect.
Two states are the same, if both states share the same node, beginning time step and ending time step.
Otherwise, two states are different.
Especially, two states with the same node but different safe intervals are different states.

\begin{figure}[htbp]
	\centering
	\includegraphics[width=\linewidth,height=2.3cm]{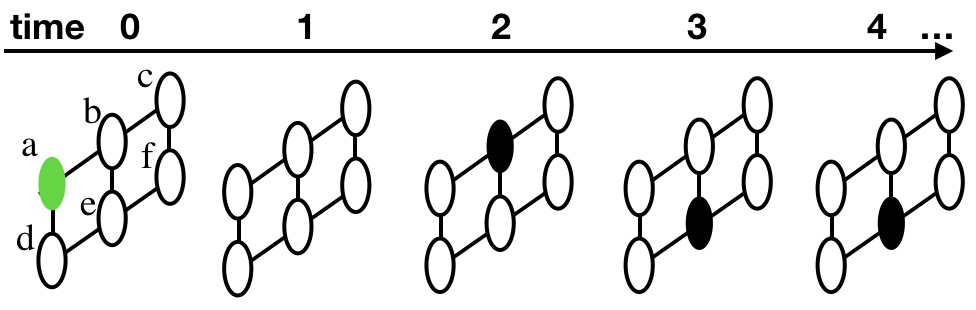}
	\caption{A toy example of SIPP over a graph $G$ with $6$ (free) nodes.
		The time-augmented graph $G^t$ is visualized with time steps between $0$ and $4$ with edges connecting different time steps omited to make the plot clear.
		A dynamic obstacle enters the environment at node $b$ at time $2$, moves to node $e$ at time $3$ and stays there afterwards.
		At node $b$, there are two possible states $(b,[0,1])$ and $(b, [3,\infty])$ while at node $e$, there is only one possible state $(e,[0,2])$.
		The agent's initial state is $(a,[0,\infty])$.
		To expand it, one successor is generated at node $d$, which is the state $(d,[0,\infty])$ with the earliest arrival time $1$, and two successors are generated at node $b$, which are $(b,[0,1])$ with the earliest arrival time $1$ and $(b, [3,\infty])$ with the earliest arrival time $3$.}
	\label{fig:sipp_eg}
\end{figure}

SIPP conducts A*-like heuristic search.
For each state $s$, let $g(s)$ represent the earliest arrival time at state $s$ at any time of the search and let $h(s)$ denote the heuristic value of $s$, which underestimates the cost-to-goal, i.e. travel time to goal, from $s$.
In addition, the $f$-value of a state is $f(s):=g(s)+h(s)$.
At any time of the search, let OPEN denote the open list containing candidate states to be expanded, where candidate states are prioritized by their $f$-values.

SIPP starts by inserting the initial state $s_o$ into OPEN ($s_o$ is a tuple of $v_{init}$ and the safe interval with a beginning time set to 0).
In each search iteration, a state with the minimum $f$-value in OPEN is popped from OPEN and expanded.
To expand a state $s=(v,[t_a,t_b])$, SIPP considers all reachable successor states from $s$ and finds the earliest possible arrival time onto each of those states via ``wait and move'' action, i.e. wait for an minimum amount of time to arrive at the successor state as early as possible.
An illustration of this expansion process can be found in Fig.~\ref{fig:sipp_eg}.
During the search, SIPP records the so-far earliest arrival time at each state $s$ as $g(s)$.
When a new trajectory is found to reach state $s$ (from $v_{init}$) with an earlier arrival time, $g(s)$ is updated.
Here, a trajectory from $v_{init}$ to some state $s=(v,[t_a,t_b])$ indicates a trajectory from $v_{init}$ to $v$ with an arrival time at $v$ within safe interval $[t_a,t_b]$.
When a state at the goal node (i.e. the node contained in the state is $v_{goal}$) is expanded, a trajectory with the minimum arrival time is found and SIPP terminates.
The key principle behind SIPP can be summarized as follows.
\begin{theorem}\label{thm:sipp}
Arriving at a state $s$ at the earliest possible time can (1) generate the maximum number of successors from $s$ and (2) find the optimal (i.e. minimum arrival time) trajectory from $v_{init}$ to $s$ during the search.
\end{theorem}

We note here that when we introduce SIPP to multi-objective settings in the subsequent sections, this principle needs to be carefully revisited.

\subsection{Multi-objective Problem Formulation}
This section follows the notations and concepts introduced in the previous section.
For a multi-objective trajectory planning problem, each edge $e\in E$ is associated with a non-negative cost vector $\vec{c}(e) \in (\mathbb{R^+})^{M}$ with $M$ being a positive integer and $\mathbb{R^+}$ being the set of non-negative real numbers.
The wait in place action takes a constant non-negative cost vector $\vec{c}_{wait}$ at any nodes in $G$.
Let $\vec{g}(\tau)$ denote the accumulated cost vector of a trajectory $\tau$, which is the sum of the cost vector of all edges present in $\tau$.
To compare any two trajectories, we compare the cost vectors between them.
Given two vectors $a$ and $b$, $a$ \textit{dominates} $b$ if every component in $a$ is no larger than the corresponding component in $b$ and there exists at least one component in $a$ that is strictly less than the corresponding component in $b$.
Formally, it is defined as follows.
\begin{definition}[Dominance \cite{ehrgott2005multicriteria}]\label{def:dom}
	Given two $M$-dimensional vectors $a$ and $b$, $a$ dominates $b$, notationally $a \succeq b$, if $\forall m \in \{1,2,\dots,M\}$, $a(m) \leq b(m)$, and there exists $m \in \{1,2,\dots,M\}$ such that $a(m) < b(m)$.
\end{definition}
If $a$ does not dominate $b$, we represent this non-dominance as $a \nsucceq b$.
Any two trajectories connecting the same pair of nodes are non-dominated to each other if the corresponding cost vectors do not dominate each other.
The goal of the multi-objective trajectory planning problem is to find $\mathcal{T}^*$, a set of all collision-free, non-dominated ($i.e.$ Pareto-optimal) trajectories with unique cost vectors.\footnote{If two trajectories have the same cost vectors, only one of them is kept.}

\vspace{-1mm}
\subsection{Algorithm}\label{sec:mosipp:alg}

As in SIPP, let a search state $s=(v,[t_a,t_b])$ be a tuple of a node and a safe interval.
In SIPP, for each closed state $s$, keeping track of just one trajectory from the start to $s$ with minimum arrival time is sufficient to compute an optimal solution to the goal vertex. In a multi-objective problem, however, there can be multiple trajectories with non-dominated cost vectors connecting $v_{init}$ and state $s$ and all of them need to be recorded at $s$ in order to compute a set of cost-unique Pareto-optimal trajectories $\mathcal{T}^*$ to the goal vertex. The algorithm also needs to be able to discriminate between those trajectories that arrive at the same state $s$ with possibly different cost vectors and arrival times.
Based on this observation, let $l=(s,\vec{g}, t_{r})$ denote a \emph{label} at state $s$, which identifies a specific trajectory from $v_{init}$ to $s$ with an arrival time $t_r$ and a cost vector $\vec{g}$.
Additionally, let $\vec{g}(l)$, $t_r(l)$ and $s(l)$ represent the cost vector, arrival time and state associated with label $l$ respectively.
Also, let $v(l)$, $t_a(l)$ and $t_b(l)$ denote the node (in $G$), beginning time and ending time of the safe interval of state $s(l)$ respectively.

As shown in Algorithm \ref{alg:mo_sipp}, multi-objective safe-interval path planning (MO-SIPP) in general has a similar workflow as SIPP (and A*).
The key differences are presented below.

 \begin{algorithm}[tb]
 	\caption{Pseudocode for MO-SIPP}\label{alg:mo_sipp}
 	\small
 	\begin{algorithmic}[1]
 		\State{$l_o\gets{(s_o,\vec{0},0)}$}
 		\State{add $l_o$ into OPEN}
 		\State{$\alpha(s) \gets \emptyset, \forall s$}
 		\State{add $l_o$ into $\alpha(s_o)$}
 		\State{$\mathcal{T}\gets \emptyset$}
 		\While{OPEN not empty} \Comment{main search loop}
 		\State{$l \gets$ OPEN.pop() }
 		\State{\textbf{if} $v(l)$ is the goal node \textbf{then}}
 		\State{\indent$\tau\gets$\text{Reconstruct($l$)}}
 		\State{\indent add $\tau$ into $\mathcal{T}$}
 		\State{\indent \text{FilterOpen($l$)}}
 		\State{\indent \textbf{continue}}
 		\State{$L_{succ} \gets$ \text{GetSuccessors}($l$) } 
 		\ForAll{$l' \in L_{succ}$ }
 		\State{\textbf{if} \text{LabelDominated($l'$)} \textbf{then}}
 		\State{\indent \textbf{continue}}
 		\State{$f(l') \gets g(l')$ + $h(s(l'))$}
 		\State{parent($l'$) $\gets l$}
 		\State{add $l'$ into OPEN}
 		\EndFor
 		\EndWhile \label{}
 		\State{\textbf{return} $\mathcal{T}$}
 	\end{algorithmic}
 \end{algorithm}

\subsubsection{Cost vectors}
Scalar values $g,h,f$ are replaced with corresponding cost vectors in MO-SIPP.
To make the notation clear, we use $\vec{g},\vec{h},\vec{f}$ to denote the cost vectors in MO-SIPP.
Specifically, $\vec{g}$ is associated with labels and it describes the cost-to-come from $v_{init}$.
$\vec{h}$ is defined over states and $\vec{h}(s)$ represents a component-wise underestimate of the cost vector of all non-dominated paths from state $s$ to the goal node.
Finally, $\vec{f}$ is defined over labels and $\vec{f}(l):=\vec{g}(l)+\vec{h}(s(l))$, which underestimates the cost vector of all trajectories connecting the start and the goal via label $l$.

\subsubsection{Label expansion}
OPEN contains all candidate labels for further expansion.
In every search iteration, a label with a non-dominated cost vector within OPEN is selected for further expansion.
The expansion step is similar to the expansion in SIPP with the only difference that MO-SIPP expands and generates new labels (rather than states as in SIPP).
Specifically, to expand a label $l$, MO-SIPP considers all the reachable states from state $s(l)$, and then for each reachable state $s'$, the earliest arrival time $t_r'$ and the cost vector $\vec{g}'$ for reaching $s'$ from $s(l)$ is computed.
A successor label $l'=(s',\vec{g}', t_r')$ is then generated.
After expanding a label $l$, a set of successor labels are obtained for comparison.

\subsubsection{Label comparison}
The comparison step in MO-SIPP differs from the one in SIPP.
In SIPP, when a new trajectory from $v_{init}$ to state $s$ is found, the $g$-value of this new trajectory and the previously stored $g$-value at $s$ is compared and the smaller value is kept. 
In MO-SIPP, multiple non-dominated trajectories, which are represented by labels, need to be tracked at state $s$.
To do so, first, a new type of dominance between labels is defined as follows.
\begin{definition}[Label-dominance]\label{def:label-dom}
	Given two labels $l=(s, \vec{g}, t_r)$ and $l'=(s', \vec{g}', t_r')$ with $s=s'$ (i.e. nodes and safe intervals in both $s$ and $s'$ are the same), if the following two conditions both hold: (i) $t_r \leq t_r'$, (ii) $\vec{g} + (t_r'-t_r)\vec{c}_{wait} \leq \vec{g}'$ (in (ii), $\leq$ means component-wise no larger than),
	then we say $l$ \emph{label-dominates} $l'$ with notation $l \succeq_l l'$. (Here, $\succeq_l$ can be intuitively interpreted as ``better than''.)
\end{definition}
Subscript $l$ in $\succeq_l$ indicates that it's a comparison between labels rather than cost vectors.
In this label-dominance relationship, if $l\succeq_l l'$, condition (i) guarantees that label $l$ will have at least the same set of successor labels as $l'$.
Condition (ii) ensures that, if there exists a path $\pi'$ from $v_o$ to $v_d$ via $l'$, then the portion of the path from $l'$ to $v_d$ can be cut-and-paste to $l$, and the resulting path $\pi''$, which connects $v_o$ and $v_d$ via $l$, will incur component-wise no larger cost than $\pi'$.
Therefore, $l'$ can be discarded.\footnote{Similar comparison rules to the label-dominance in this work have been developed in GSIPP~\cite{gonzalez2012using} for a single-objective case, where a single non-arrival-time objective is minimized. The label-dominance here can be regarded as a generalization of the rule in GSIPP~\cite{gonzalez2012using} from single-objective to multi-objective.}

\subsubsection{Frontier Sets}
To keep track of multiple non-dominated trajectories at a state $s$, the notion of frontier set $\alpha(s)$ is introduced, which denote a set of non-dominated labels at $s$.
To initialize, $\alpha(s)$ for all states but the initial state are set to an empty set and $\alpha(s_o)$ is set to be a set containing $(s_o, \vec{0}, 0)$ only.
With that in hand, we now introduce the comparison procedure, as shown in Algorithm.~\ref{alg:compare}.
When a new label $l'$ is generated at state $s$ (i.e. $s(l')=s$), it is compared with every label $l \in \alpha(s)$.
If none of the labels in $\alpha(s)$ label-dominates $l'$, then $l'$ is inserted into $\alpha(s)$ and $l'$ is used to filter $\alpha(s)$, which removes all labels $l \in \alpha(s)$ that are label-dominated by $l'$.
By doing so, at any time of the search, $\alpha(s)$ maintains a set of labels at state $s$ with either a non-dominated cost vector or an earlier arrival time.
Finally, if a generated label is not label-dominated, it is inserted into OPEN for further expansion.

\begin{algorithm}[tbp]
	\caption{Pseudocode for LabelDominated($l'$)}\label{alg:compare}
 	\small
	\begin{algorithmic}[1]
		\State{$l'$ is the input label to be compared.}
		\ForAll{$l \in \alpha(s(l'))$}
		\State{\textbf{if} $ l \succ_l l'$ \textbf{then}}
		\State{\indent\Return true \Comment should be discarded}
		\EndFor
		\ForAll{$l \in \alpha(s(l'))$}
		\State{\textbf{if} $ l' \succ_l l$ \textbf{then}}
		\State{\indent remove $l$ from $\alpha(s(l'))$}
		\State{\indent remove $l$ from OPEN if OPEN contains $l$.}
		\EndFor
		\State{add $l'$ to $\alpha(s(l'))$}
		\State{\Return false \Comment should not be discarded}
	\end{algorithmic}
\end{algorithm}

\subsubsection{Filtering and termination}
During the search, when a label $l$ with $v(l)$ being the goal node is popped from OPEN, a Pareto-optimal trajectory is found and is inserted into $\mathcal{T}$, which is a set that contains all Pareto-optimal trajectories found during the search.
The trajectory represented by a label can be easily reconstructed by iteratively backtracking the parent pointers of labels.
Different from SIPP, what's new in MO-SIPP is that the cost vector $\vec{g}(l)$ is used to filter OPEN, which removes all candidate labels $l'$ in OPEN if $\vec{g}(l) \succeq \vec{g}(l')$ or $\vec{g}(l) = \vec{g}(l')$.
The intuition behind this filtering is that, any filtered candidate labels can not be part of a Pareto-optimal trajectory and is thus discarded, as the cost vectors of all edges in $G^t$ are non-negative.
MO-SIPP terminates when OPEN is empty, which guarantees that $\mathcal{T}$ contains all cost-unique Pareto-optimal trajectories.

\vspace{-1mm}
\subsection{Analysis}
In this section, we show that MO-SIPP is able to compute a set of cost-unique Pareto-optimal trajectories $\mathcal{T}^*$.
\begin{lemma}\label{lem:max_succ}
	Arriving at a state $s$ at the earliest possible time can generate the maximum number of successors.
\end{lemma}
\begin{proof} Note that the set of reachable successors from a label corresponding to a state depends only on the arrival time of the label and is not dependent on the cost vector of the label. Therefore, given two labels $l=(s,\vec{g},t_r)$ and $l'=(s,\vec{g}',t'_r)$ at the same state $s=(v,[t_a,t_b])$ with $t_r \leq t'_r$, any reachable state from $l'$ (within time interval $[t'_r,t_b]$) is also reachable from $l$ (within time interval $[t_r,t_b]$), regardless of the cost vectors $\vec{g}$ or $\vec{g}'$. Hence proved.
\end{proof}
\vspace{1mm}

The second property of SIPP  in Theorem  \ref{thm:sipp} does not apply to MO-SIPP because SIPP aims to optimize the arrival time corresponding to a single objective while MO-SIPP aims to find non-dominated trajectories corresponding to multiple objectives. With the modified label comparison procedure in MO-SIPP, the following lemma holds.

\vspace{1mm}
\begin{lemma}\label{lem:dom_prune}
	At any time of the search, if a label at state $s$ is pruned, the trajectory represented by the label cannot be part of a cost-unique Pareto-optimal trajectory.
\end{lemma}
\begin{proof}
In MO-SIPP, there are three cases where a label $l'$ is pruned:
\begin{itemize}
	\item $l'$ is filtered by the FilterOpen procedure;
	\item $l'$ is label-dominated by an existing label $l\in \alpha(s(l'))$ (line 3 in Algorithm~\ref{alg:compare});
	\item $l'$ is label-dominated by a label $l$ that enters $\alpha(s(l'))$  (line 6 in Algorithm~\ref{alg:compare}). 
\end{itemize}
For either of those three cases, $\vec{g}(l')$ is dominated by or equal to the cost vector of some other labels expanded or to be expanded.
In addition, if a label $l'$ is label-dominated by $l$, any possible successors of $l'$ is also reachable from $l$ since $t_r(l) \leq t_r(l')$.
Therefore, $l'$ cannot be part of a cost-unique Pareto-optimal trajectory.
\end{proof}
\vspace{1mm}

Therefore, in MO-SIPP, label expansion always generates successors with the earliest arrival time which guarantees the maximum number of successors (Lemma~\ref{lem:max_succ}).
Those generated successor labels are only pruned if they cannot lead to a cost-unique Pareto-optimal trajectory (Lemma~\ref{lem:dom_prune}).
MO-SIPP terminates when OPEN is empty, which means all labels are expanded or pruned, which computes a $\mathcal{T}^*$. This property can be summarized with the following theorem:

\vspace{1mm}
\begin{theorem}\label{thm:mo_sipp}
	MO-SIPP algorithm is able to compute all cost-unique Pareto-optimal trajectories connecting the start and the goal.
\end{theorem}

	\section{Multi-objective Conflict-based Search}\label{sec:mocbssipp}
	\ifthenelse{\boolean{shortver}}{%
		We refer the reader to \cite{ren2021multi,ren21mosipp} for a detailed description about multi-objective multi-agent path finding (MOMAPF) problem and the basic multi-objective conflict-based search (MO-CBS) algorithm.
In this section, we focus on how to integrate the developed MO-SIPP as the low level planner for MO-CBS and MO-CBS-t algorithms~\cite{ren2021multi}.

The low level search in MO-CBS requires a single-agent multi-objective planner that can find all cost-unique (individual) Pareto-optimal paths for a single-agent while satisfying a set of constraints.
The set of constraints includes both a set of node constraints $\{(v,t)\}$, where the agent is prevented from entering node $v$ at time $t$, and a set of edge constraints $\{(e,t)\}$, where agent is prevented from moving through edge $e$ at time $t$.
In MO-CBS (and MO-CBS-t), NAMOA* is used as the low level planner to search $G^t$, where constraints are represented as blocked nodes and edges in $G^t$.

To use MO-SIPP as the low level planner, node constraints $\{(v,t)\}, v\in V$ can be directly considered as nodes in $G^t$ that are occupied by some dynamic obstacles.
Edge constraints $\{(e,t), e\in E\}$ can be maintained as a lookup table so that during label expansion, MO-SIPP avoids using edge $e$ at time $t$ when generating successors with the minimum arrival time.
Additionally, when a label $l$ at the goal node is expanded, we also need to check whether there exists a node constraint $(v,t)$ with $t > t_r(l)$.
\begin{itemize}
	\item If there exists such a node constraint, it indicates that agent can not stay at node $v(l)$ after the arrival because $v(l)$ is blocked at some future time step. In this case, MO-SIPP has not yet found a Pareto-optimal trajectory and this label will be further expanded like other candidate labels.
	\item Otherwise, MO-SIPP finds a Pareto-optimal trajectory.
\end{itemize}

Finally, as MO-SIPP is able to compute all cost-unique Pareto-optimal trajectories, (like the original NAMOA* in $G^t$), the property of MO-CBS is not affected when MO-SIPP is embedded as the low level planner.

	}{
		
In this section, we first review the definition of multi-objective multi-agent path finding (MOMAPF) problem and then present how to embed MO-SIPP as the low level planner in MO-CBS.

\subsection{MOMAPF Problem Description}

Let index set $I=\{1,2,\dots,N\}$ denote a set of $N$ agents. 
All agents move in a workspace represented as a finite graph $G=(V,E)$, where the vertex set $V$ represents all possible locations of agents and the edge set $E = V \times V$ denotes the set of all the possible actions that can move an agent between any two vertices in $V$.
An edge between two vertices $u, v \in V$ is denoted as $(u,v)\in E$ and the cost of an edge $e \in E$ is a $M$-dimensional non-negative vector cost$(e) \in (\mathbb{R}^{+})^M\backslash\{0\}$ 
with $M$ being a positive integer.

Here, we use a superscript $i \in I$ over a variable to represent the specific agent that the variable belongs to (e.g. $v^i\in V$ means a vertex with respect to agent $i$). 
Let $\pi^i(v^i_{1}, v^i_{\ell})$ be a path that connects vertices $v^i_{1}$ and $v^i_{\ell}$ via a sequence of vertices $(v^i_{1},v^i_{{2}},\dots,v^i_{\ell})$ in the graph $G$. 
Let $g^i(\pi^i(v^i_{1}, v^i_\ell))$ denote the $M$-dimensional cost vector associated with the path, which is the sum of the cost vectors of all the edges present in the path, $i.e.$, $g^i(\pi^i(v^i_{1}, v^i_{\ell})) = \Sigma_{j=1,2,\dots,{\ell-1}} \text{cost}(v^i_{{j}}, v^i_{{j+1}})$. 

All agents share a global clock and all the agents start their paths at time $t=0$. Each action, either wait or move, for any agent requires one unit of time. 
Any two agents $i,j \in I$ are said to be in conflict if one of the following two cases happens. The first case is a ``vertex conflict'' where two agents occupy the same location at the same time. The second case is an ``edge conflict'' (also known as ``swap conflict'' in the literature) where two agents move through the same edge from opposite directions at times $t$ and $t+1$ for some $t$.   

Let $v_o^i, v^i_f \in V$ respectively denote the initial location and the destination of agent $i$.
Without loss of generality, to simplify the notations, we also refer to a path $\pi^i(v^i_{o}, v^i_{f})$ for agent $i$ between its initial location and destination as simply $\pi^i$. Let $\pi=(\pi^1,\pi^2,\dots, \pi^N)$ represent a joint path for all the agents.
The cost vector of this joint path is defined as the vector sum of the individual path costs over all the agents, $i.e.$, $g(\pi) = \Sigma_i g^i(\pi^i)$.

To compare any two joint paths, we compare the cost vectors corresponding to them using the dominance defined in Def.~\ref{def:dom}.
Any two joint paths are non-dominated if the corresponding cost vectors do not dominate each other. The set of all non-dominated conflict-free joint paths is called the {\it Pareto-optimal} set.
This work aims to find any maximal subset of the Pareto-optimal set, where any two joint paths in this subset do not have the same cost vector.

\subsection{A Brief Review of MO-CBS}

Multi-objective conflict-based search (MO-CBS) is a two-level search algorithm that begins by computing a set of all cost-unique individual Pareto-optimal paths $\Pi^i_o, \forall i\in I$ for each agent independently via a single-agent multi-objective path planner, such as NAMOA*~\cite{mandow2008multiobjective}.
By taking the combination $\Pi_o = \Pi^1_o\times \Pi^2_o \times \dots \times \Pi^N_o$, an initial set of joint paths is generated.
For each joint path $\pi_o \in \Pi_o$, a corresponding high-level search root node, which contains $\pi_o$, the cost vector of $\pi_o$ and an empty constraint set, is generated and inserted into OPEN.
In MO-CBS, OPEN is a list containing all candidate high-level nodes.

In each high-level search iteration, a candidate node $P$ with a non-dominated cost vector within OPEN is popped and checked for conflict along every pair of individual paths.
For the first detected conflict between some pair of agents $i,j$, MO-CBS splits the conflict and generates two constraints, where one constraint is imposed on agent $i$ while the other constraint is added to agent $j$.
Next, for both constraints, a low level search, which is a single-agent multi-objective planner, is invoked to compute the individual Pareto-optimal paths $\Pi^i$ for agent $i$ (and $j$) while satisfying all the constraints added for agent $i$  (this can be found by iterative backtracking the constraints of all ancestor high-level nodes).
For each $\pi^i \in \Pi^i$, a corresponding high-level node $P'$ is generated from $P$ and is inserted into OPEN, where the joint path in $P'$ is copied from $P$ with agent $i$'s individual path replaced with $\pi^i$.

During the high-level search of MO-CBS, if a conflict-free joint path $\pi$ is found, 
MO-CBS uses the cost vector of $\pi$ to filter candidates nodes in OPEN, which removes high-level nodes with dominated cost vectors.
In addition, MO-CBS uses the cost vector of $\pi$ to filter $\mathcal{S}$, which is a set of collision-free joint paths computed so far, and then add $\pi$ into $\mathcal{S}$.
MO-CBS terminates when OPEN is empty.
At termination, MO-CBS is guaranteed to find all cost-unique Pareto-optimal joint paths.

Due to the (possibly) large size of $\Pi_o$ at initialization, a variant of MO-CBS that employs a tree-wise expansion strategy (MO-CBS-t) is introduced in~\cite{ren2021multi} to overcome this difficulty.
In MO-CBS-t, root nodes are generated on demand and one root node is exhaustively searched until OPEN depletes before the next root is generated.
When all roots are generated and searched, MO-CBS-t terminates and finds all cost-unique Pareto-optimal joint paths.
MO-CBS-t also enjoys the benefits of identifying the first feasible joint path quickly.
Numerical results~\cite{ren2021multi} show that the first feasible joint path found by MO-CBS is typically Pareto-optimal or very close to be Pareto-optimal.

\subsection{Using MO-SIPP as the Low Level Planner}
The low level search in MO-CBS requires a single-agent multi-objective planner that can find all cost-unique (individual) Pareto-optimal paths for a single-agent while satisfying a set of constraints.
The set of constraints includes both a set of node constraints $\{(v,t)\}$, where the agent is prevented from entering node $v$ at time $t$, and a set of edge constraints $\{(e,t)\}$, where agent is prevented from moving through edge $e$ at time $t$.
In MO-CBS (and MO-CBS-t), NAMOA* is used as a low level planner to search a time-augmented graph $G^t$, where constraints are represented as blocked nodes and edges in $G^t$.

To use MO-SIPP as the low level planner, node constraints $\{(v,t)\}, v\in V$ can be directly considered as nodes in $G^t$ that are occupied by some dynamic obstacles.
Edge constraints $\{(e,t), e\in E\}$ can be maintained as a lookup table so that during label expansion, MO-SIPP avoids using edge $e$ at time $t$ when generating successors with the minimum arrival time.
Additionally, when a label $l$ at the goal node is expanded, we also need to check whether there exists a node constraint $(v,t)$ with $t > t_r(l)$.
\begin{itemize}
	\item If there exists such a node constraint, it indicates that agent can not stay at node $v(l)$ after the arrival because $v(l)$ is blocked at some future time step. In this case, MO-SIPP has not yet found a Pareto-optimal trajectory and this label will be further expanded like other candidate labels.
	\item Otherwise, MO-SIPP has found a Pareto-optimal trajectory.
\end{itemize}

Finally, as MO-SIPP is able to compute all cost-unique Pareto-optimal trajectories, (like the original NAMOA* in $G^t$), the property of MO-CBS is not affected when MO-SIPP is embedded as the low level planner.

	}

	\section{Numerical Results}\label{sec:result}
	
\subsection{Test Settings and Implementation}
Our test settings follows the ones in MO-CBS~\cite{ren2021multi}.
We selected four (grid) maps from \cite{stern2019multi} and generated an un-directed graph $G$ by making each grid four-connected.
To assign cost vectors to edges in $G$, we first assigned every agent, a cost vector $a^i, \forall i \in I$ of length $M$ (the number of objectives) and assigned every edge $e$ in $G$ a scaling vector $b(e)$ of length $M$, where each element in both $a^i$ and $b(e)$ were randomly sampled from integers in $[1, 10]$.
The range $[1,10]$ follows the convention used in \cite{mandow2005new}.
The cost vector for agent $i$ to go through $e$ is the component-wise product of $a^i$ and $b(e)$. If agent $i$ wait in place, the cost incurred is $a^i$ per time step.
We use unit vector scaled by Manhattan distance between each node $u\in G$ and the goal node as the heuristic vector for any states $s$ with $v(s)=u$.
We tested the algorithms by varying the number of objectives ($M$) and the number of agents ($N$) within a run time limit of {\it five} minutes.
Our comparison involves MOM*~\cite{ren2021subdimensional}, MO-CBS-t~\cite{ren2021multi} and MO-CBS-ts, which ``s'' stands for MO-SIPP.
All algorithms are implemented in Python and compared on a computer with an Intel Core i7 CPU and 16GB RAM. 

\begin{figure}[tb]
	\centering
	\includegraphics[width=\linewidth]{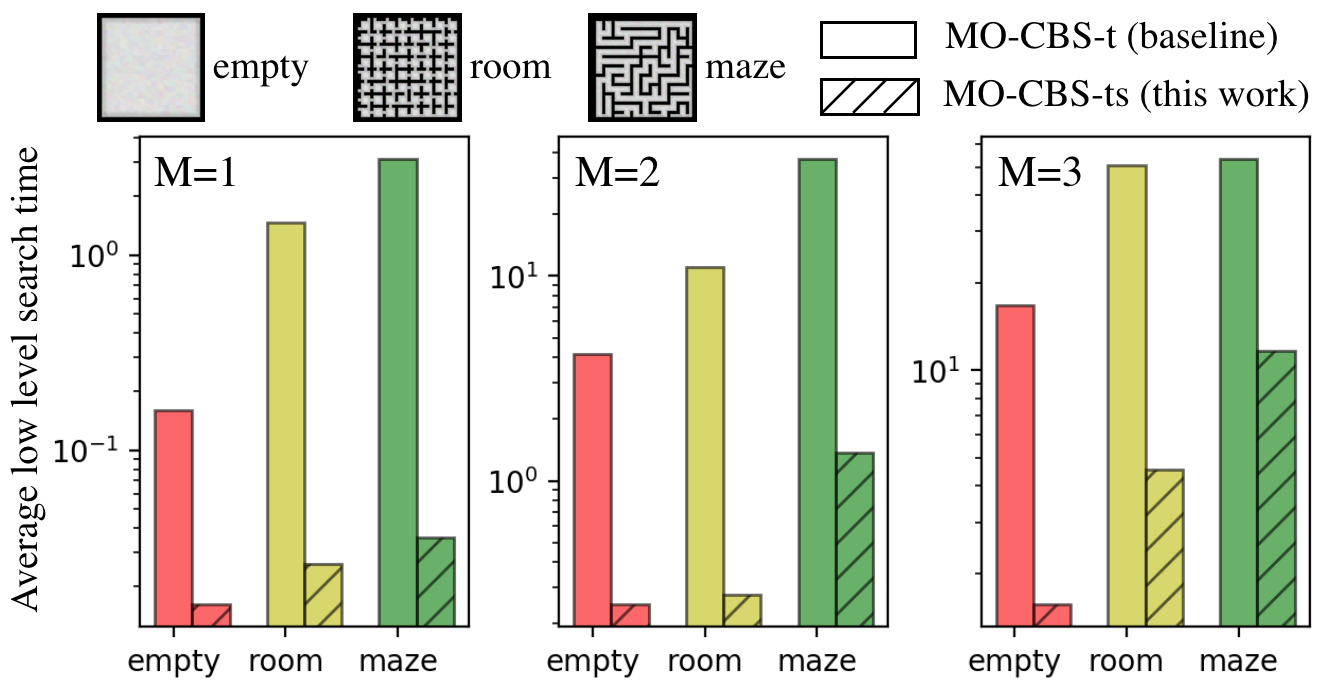}
	\caption{Comparing the average low level search time with number of objectives $M=1,2,3$ in different maps, while number of agents $N=2$ is fixed.}
	\label{fig:low_level_avg}
\end{figure}

\begin{figure}[h!]
	\centering
	\includegraphics[width=\linewidth]{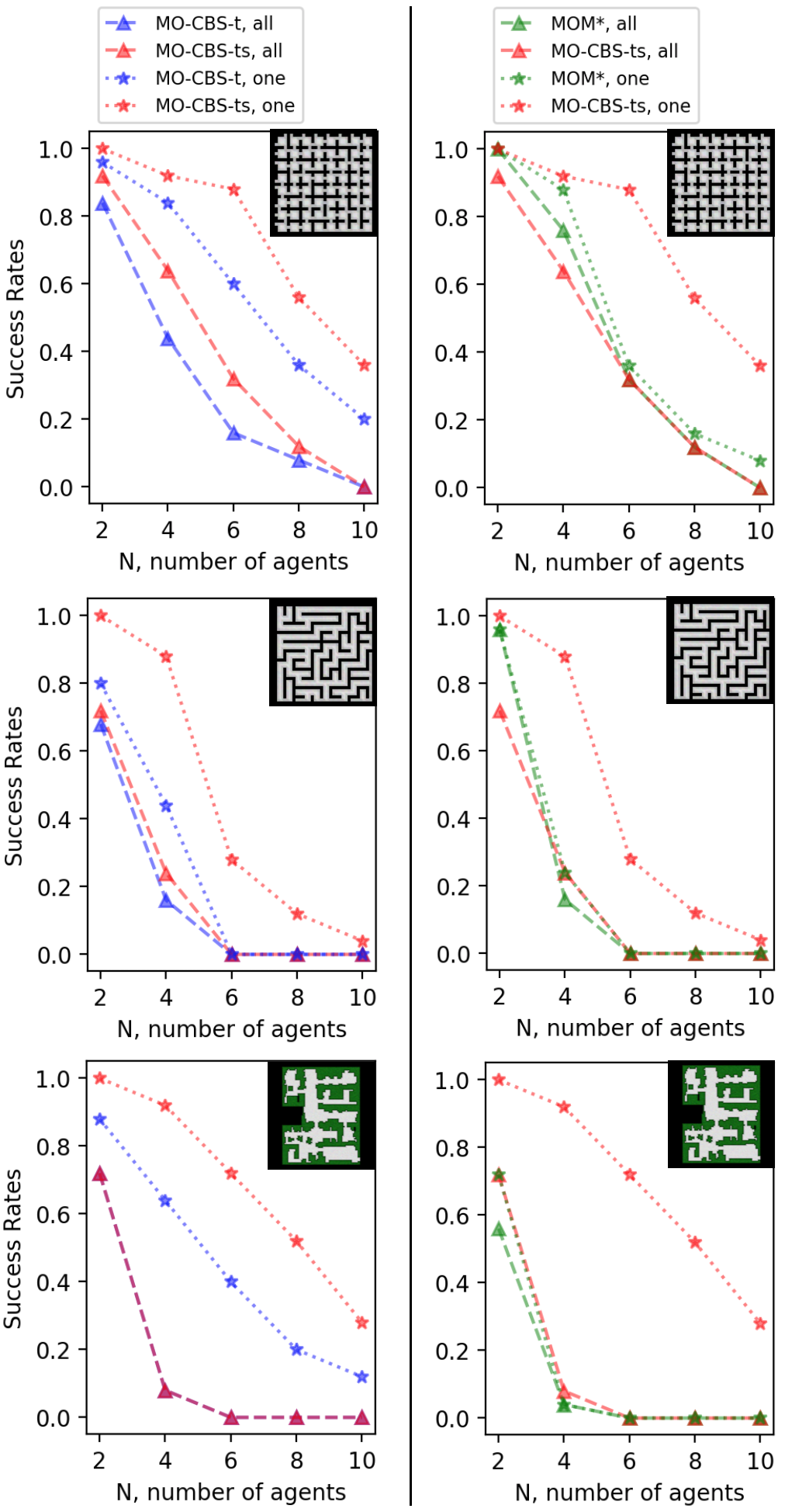}
	\caption{Success rates of MO-CBS-t (baseline), MOM* (baseline) and MO-CBS-ts (proposed) about (1) finding all cost-unique Pareto-optimal joint paths and (2) finding at least one feasible joint path, within a time limit of five minutes in different maps with a varying $N$ and a fixed $M=2$. Left column compares MO-CBS-t and MO-CBS-ts while the right column compares MOM* and MO-CBS-ts.}
	\label{fig:mocbs_succes_rates}
\end{figure}

\subsection{Low Level Search Comparison}\label{sec:result:low_level}

Both MO-CBS and MO-CBS-t use the same low level planner and MO-CBS-t serves as a baseline for comparison.
In both MO-CBS-t and MO-CBS-ts, all high-level nodes with non-dominated cost vectors in OPEN are prioritized in lexicographic order, which enforces a deterministic and same order for expansion in both algorithms.
For each test instance, the average low level search time per call $\bar{t}_{instance}$ is computed for both MO-CBS-t and MO-CBS-ts.
Next, for each map, $\bar{t}_{instance}$ is averaged over all instances and this average (denoted as $\bar{t}_{map}$) is plotted in Fig.~\ref{fig:low_level_avg}.

Fig.~\ref{fig:low_level_avg} shows $\bar{t}_{map}$ in empty, room and maze like maps with $M$ varied from 1 to 3 and $N$ fixed at 2.
In all maps, the low level search in MO-CBS-ts (this work) runs much faster than the low level search in MO-CBS-t (baseline). In general, we observed an order of magnitude improvement in the low level search time for $M=1,2,3$.
For example, for $M=2$ and the room map, the low level search in MO-CBS-t requires about 10 seconds on average while MO-CBS-ts requires less than one second.

\ifthenelse{\boolean{shortver}}{%
}
{
	\begin{figure}[h!]
		\centering
		\includegraphics[width=0.9\linewidth]{risk_map.png}
		\caption{(Left) Risk model. (Right) A risk map where black cells represents semi-constructed architecture and while cells are completely safe cells where risk scores are zero. Grey cells have non-zero risk scores and darker color indicates higher risk score.}
		\label{fig:risk_map}
		\vspace{-2mm}
	\end{figure}
	
	\begin{figure*}[h!]
		\centering
		\includegraphics[width=\linewidth]{solution_construction.png}
		\caption{Leftmost plot shows the Pareto-optimal front of the construction site example. The three plots on the right show the joint path of agents corresponding to the red, green and orange solution respectively.}
		\label{fig:solution_construction}
		\vspace{-2mm}
	\end{figure*}
}

\subsection{Success Rates Comparison}

The proposed MO-CBS-ts is compared with MO-CBS-t in terms of (1) success rates of finding all cost-unique Pareto-optimal joint paths and (2) success rates of finding at least one feasible joint path, within the five minutes time limit.
Note that the first solution computed by both MO-CBS-t and MO-CBS-ts is not guaranteed to be Pareto-optimal optimal but has been shown to be empirically near Pareto-optimal~\cite{ren2021multi}.
Additionally, during the test after finding the first feasible solution, both MO-CBS-t and MO-CBS-ts keep running in pursuit of the entire cost-unique Pareto-optimal set.
Here, $M=2$ is fixed and $N$ varies.
As shown in the left column in Fig.~\ref{fig:mocbs_succes_rates}, in terms of metric (1), the proposed MO-CBS-ts outperforms MO-CBS-t in the room setting and performs no worse than MO-CBS-t in the other two settings.
When $N=4,6$, in the room map, the success rates of (1) are improved by $\approx$ 20\%.
In terms of metric (2), MO-CBS-ts outperforms MO-CBS-t in all settings.
When $N=4$, in the maze map, the success rate of (2) is improved by $\approx$ 40\%.
The results show that, with an improved low level planner, the proposed MO-CBS-ts outperforms MO-CBS-t.
The improvement in metric (2) indicates that the search process of MO-CBS is sped up by using MO-SIPP since both MO-CBS-t and MO-CBS-ts are enforced the same expansion order on the high level (as explained in Sec.~\ref{sec:result:low_level}) and MO-CBS-ts is more likely to find the first solution.

Additionally, it's a bit surprising that the significant improvement in the low level search (Fig.~\ref{fig:low_level_avg}) leads to only a moderate or modest improvement in the success rate of finding all cost-unique Pareto-optimal solutions.
The main reason lies in the complexity of the high level search in MO-CBS~\cite{ren2021multi} and the enormous size of the Pareto-optimal set~\cite{ren2021subdimensional}.
This result implies the necessity to improve the high level search in MO-CBS along with the low level improvement, which is planned as our next step (Sec.~\ref{sec:conclude}).

\subsection{Comparison with MOM*}
MOM*~\cite{ren2021subdimensional} is used as another baseline for comparison.
As shown in the right column in Fig.~\ref{fig:mocbs_succes_rates}, with a fixed $M=2$ and a varying $N$, in terms of the success rates of finding at least one feasible joint path, MO-CBS-ts outperforms MOM* in all maps.
To find all cost-unique Pareto-optimal joint paths, there is no algorithm that outperforms the other in all settings.

	\section{Conclusion}\label{sec:conclude}
	
This article considers the problem of multi-objective multi-agent path finding (MOMAPF).
We first develop a multi-objective version of the well-known safe-interval path planning (SIPP) algorithm named MO-SIPP and show that MO-SIPP computes all cost-unique Pareto-optimal trajectories connecting a given start and goal location in the presence of dynamic obstacles.
We then combine MO-SIPP with MO-CBS and propose a new algorithm called MO-CBS-ts for MOMAPF.
Our numerical results show that MO-CBS-ts significantly improves the average low level search time by an order of magnitude and improves the overall success rates in general.
Although the proposed MO-SIPP is presented as the low level planner for MO-CBS, the MO-SIPP is a general single-agent multi-objective planner and can be applied to other applications as well, when Pareto-optimal trajectories subject to multiple objectives is required.


There are several possible directions for future work.
First, one can consider improving the high level search of MO-CBS by leveraging techniques designed for (single-objective) CBS (such as \cite{boyarski2015icbs, li2019disjoint}). Besides, one can also consider develop a sub-optimal version of MO-CBS that approximates the Pareto-optimal set with guarantees by leveraging (single-objective) bounded sub-optimal CBS~\cite{barer2014suboptimal, li2020eecbs}.
	
    \section*{Acknowledgment}
    This material is based upon work supported by the National Science Foundation under Grant No. 2120219 and 2120529. Any opinions, findings, and conclusions or recommendations expressed in this material are those of the author(s) and do not necessarily reflect the views of the National Science Foundation.
%
%
%
	
	\bibliographystyle{plain}
	\bibliography{references}

\end{document}